\documentclass[letterpaper, 10pt, conference]{ieeeconf}

\IEEEoverridecommandlockouts
\usepackage{amsmath,amssymb}

\usepackage{amsthm}
\usepackage{bm}
\usepackage{graphicx}
\usepackage{booktabs}
\usepackage{url}
\usepackage{algorithm}
\usepackage{algpseudocode}

\theoremstyle{plain}
\newtheorem{theorem}{Theorem}
\newtheorem{lemma}[theorem]{Lemma}
\newtheorem{corollary}[theorem]{Corollary}
\newtheorem{proposition}[theorem]{Proposition}

\theoremstyle{definition}
\newtheorem{definition}{Definition}
\newtheorem{assumption}{Assumption}
\newtheorem{remark}{Remark}

\newcommand{\R}{\mathbb{R}}
\newcommand{\bfc}{\bm{c}}
\newcommand{\bff}{\bm{f}}
\newcommand{\bfg}{\bm{g}}
\newcommand{\bfp}{\bm{p}}
\newcommand{\bfn}{\bm{n}}
\newcommand{\bfr}{\bm{r}}
\newcommand{\bfF}{\bm{F}}
\newcommand{\bftau}{\bm{\tau}}
\newcommand{\bfH}{\bm{H}}
\newcommand{\HdotG}{\dot{\bm{H}}_G}
\newcommand{\HdotGz}{\dot{\bm{H}}_{G,0}}
\newcommand{\norm}[1]{\left\|#1\right\|}
\newcommand{\bfJ}{\bm{J}}
\newcommand{\bfU}{\bm{U}}
\newcommand{\bfV}{\bm{V}}
\newcommand{\bfu}{\bm{u}}
\newcommand{\bfD}{\bm{D}}
\newcommand{\bfSigma}{\bm{\Sigma}}
\newcommand{\bfx}{\bm{x}}
\newcommand{\bfz}{\bm{z}}

\title{\LARGE \bf
On the Emergence of Pendular Structure\\
in Multi-Contact Locomotion
}

\author{Lingxue Lyu$^{*}$ \quad Zihui Liu%
\thanks{$^{*}$Lingxue Lyu is with the School of Engineering and Applied
Science, University of Pennsylvania, Philadelphia, PA 19104, USA
{\tt\small lingxuelyu@alumni.upenn.edu}}%
\thanks{Zihui Liu is with the Department of Aeronautics \& Astronautics,
Stanford University, Stanford, CA 94305, USA
{\tt\small zl90@alumni.stanford.edu}}
}

\begin{document}

\maketitle
\thispagestyle{empty}
\pagestyle{empty}
\nocite{IEEEtran:BSTCTL}

\begin{abstract}
LIPM is everywhere in legged-locomotion control, but almost always as a
modeling choice rather than as something the controller's cost actually
prefers.
This note tries to make that link more explicit.
Working from a small centroidal OCP that penalizes the rate of angular
momentum, we look at what its optimum tends to look like.
Three things come out.
With full-rank stance, the optimum drifts toward a pendular force pattern at
a rate determined by the SVD of the moment Jacobian; the constant is set by
foot-span geometry and matches the experiments to within ${\sim}16\%$.
With $N{=}2$ stance, as in trot, the friction cone introduces a lower bound
on $\|\HdotG\|$ that no amount of weight tuning fixes; we also see a
non-smooth feasibility kink at a critical horizontal acceleration that we
can write in closed form.
Adding a task term that asks for a nonzero $\HdotG$ moves the optimum off
the pendular set in a predictable way.
None of this is far from the classical ZMP/DCM picture.
We test these claims on a point-mass quadruped and on the Unitree Go1 in
MuJoCo (open-loop QP and a torque-level closed-loop controller), and we
note where the asymptotic story stops being a good description of what the
closed loop actually does.
\end{abstract}

\section{INTRODUCTION}
\label{sec:intro}

LIPM~\cite{kajita2001} is everywhere.
It shows up as a building block in biped
work~\cite{kajita2003,wieber2006,pratt2006,englsberger2015} and in quadruped
MPC~\cite{dicarlo2018,bledt2018}, in trajectory generation, in
capture-region analyses, and in countless control papers in between.
Mostly it is just \emph{assumed}: the CoM moves along a fixed-height surface
above a virtual pivot, and the rest of the controller is built on top.
There are extensions that drop one assumption or another---variable
height~\cite{hopkins2014}, non-coplanar contacts~\cite{caron2015},
angular-momentum corrections~\cite{englsberger2015,koolen2012}---but the
pendular form is still added by hand.
Chen and Posa~\cite{chen2021} ask the more general question of which
reduced-order models are task-optimal.
We have a narrower question: what kind of objective makes the pendular form
fall out, and where does the picture break.

The other thread we want to flag is the body of work that uses optimal
control directly on multi-contact behaviors---hand-contact fall
recovery~\cite{wang2017realtime,wang2018realization}, contact-transition
trees~\cite{wang2018unified}, balance under heavy
underactuation~\cite{sharma2016bicycle}, and runtime dynamics
calibration~\cite{wang2021calibration}.
The optima in those settings often \emph{look} pendular, even when nothing
in the formulation is asking them to.
We do not have a complete answer for why, but the pieces below at least
explain part of it.

\begin{figure}[!t]
  \centering
  \includegraphics[width=\columnwidth]{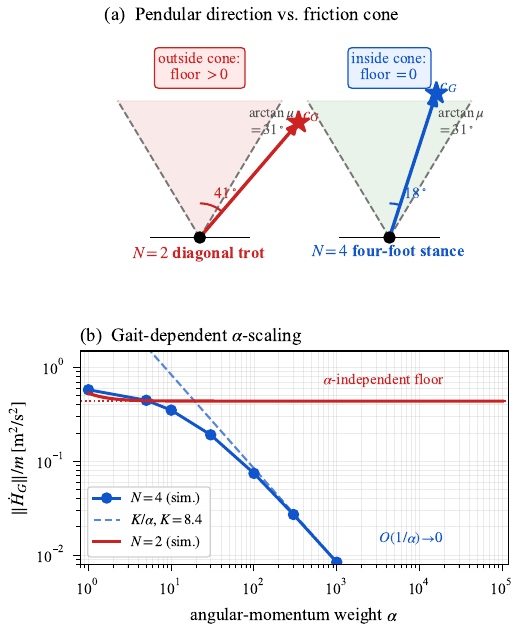}
  \caption{(a) Pendular direction (foot$\to$CoM) versus the friction cone of
  half-angle $\arctan\mu=31^\circ$ at $\mu=0.6$.
  Diagonal trot ($N=2$): the pendular direction tilts $40.6^\circ$ from
  vertical, \emph{outside} the cone, so $\|\HdotG\|$ has an $\alpha$-independent
  floor (Thm.~\ref{thm:floor}).
  Four-foot stance ($N=4$): tilt is $\sim\!18^\circ$, \emph{inside} the cone,
  floor is zero.
  (b) Go1 numerical $\|\HdotG\|/m$ vs.\ $\alpha$: $N=4$ follows $K/\alpha$
  with $K=8.4$ (Thm.~\ref{thm:emerge}); $N=2$ saturates at $0.44~\mathrm{m^2\!/s^2}$.}
  \label{fig:concept}
\end{figure}

The setup is a small centroidal OCP with a quadratic penalty on $\|\HdotG\|$,
the usual contact-force regularizer, a normal-acceleration term, and (if we
want a non-balance task) an explicit $\HdotG$-tracking term.
Three things show up (Fig.~\ref{fig:concept}).
For full-rank stance the optimum drifts toward the pendular force pattern at
$O(1/\alpha)$, with a constant we can read off the moment-Jacobian's SVD;
that constant turns out to just be the foot-span geometry.
For $N=2$ stance the situation is qualitatively different.
The moment operator drops to rank $1$, so a $\hat\bfD$-aligned residual
remains for any force distribution, and on top of that the friction cone
imposes its own lower bound that becomes binding at a finite horizontal
acceleration.
Adding the task term, finally, just slides the optimum off the manifold by a
factor $\lambda/(\alpha+\lambda)$.
We are not claiming any of this is far from
ZMP~\cite{vukobratovic1972}/DCM~\cite{englsberger2015,koolen2012} thinking;
the goal here is to write the connection down explicitly so that the
constants and limits become checkable.

\section{PROBLEM FORMULATION}
\label{sec:formulation}

Let $\bfc\in\R^3$ be the CoM, $m>0$ the total mass, $\bfg=[0,0,-g]^\top$.
With $N_c\geq 1$ contact points $\bfr_i$ exerting forces $\bff_i$, the
centroidal dynamics~\cite{orin2013} read
\begin{align}
  m\ddot{\bfc} &= \textstyle\sum_{i}\bff_i + m\bfg, \label{eq:newton}\\
  \HdotG       &= \textstyle\sum_{i}(\bfr_i-\bfc)\times\bff_i, \label{eq:euler}
\end{align}
so the full-body dynamics affect the CoM only through the net wrench.
Each contact obeys friction
$\bff_i\in\mathcal{K}_i=\{\bff:f_{n,i}\geq 0,\,\|\bff_{t,i}\|\leq\mu_i f_{n,i}\}$,
and the ZMP is $\bfp_\mathrm{ZMP}=\bfc_{xy}-(h/g)\ddot{\bfc}_{xy}$
\cite{vukobratovic1972}.

\begin{definition}[Net-Wrench OCP]
\label{def:ocp}
For $\alpha,\beta,\gamma\geq 0$, $\lambda\geq 0$, target
$\HdotG^\mathrm{task}(t)$, and support normal $\bfn\in\mathbb{S}^2$:
\begin{align}
  \min_{\bfc(\cdot),\,\{\bff_i(\cdot)\}}\;
  \int_0^T\!\!\Bigl[
    & \alpha\norm{\HdotG}^2
      + \lambda\norm{\HdotG-\HdotG^\mathrm{task}}^2 \nonumber\\
    & + \beta(\bfn^\top\ddot{\bfc})^2
      + \gamma\textstyle\sum_i\norm{\bff_i}^2
  \Bigr]dt \label{eq:ocp}
\end{align}
subject to~\eqref{eq:newton}, $\bff_i\in\mathcal{K}_i$, and boundary conditions.
The $\lambda=0$ case is the \emph{balance-dominated} OCP; $\lambda>0$ is
\emph{task-dominated}.
\end{definition}

\begin{definition}[Pendular Manifold]
\label{def:pendular}
$\mathcal{P}(\bfn,h)=\{(\bfc,\bfF):\HdotG=\bm{0},\;\bfn^\top(\bfc-\bfp)=h,\;
\bfF=(mg/h)(\bfc-\bfp),\;\bfp\in\Pi\}$.
\end{definition}

\begin{assumption}[Feasibility/Regularity]
\label{ass:both}
The OCP is feasible with bounded cost, friction cones are non-empty and
convex, and $h^*=\bfn^\top(\bfc^*-\bfp^*)\geq h_{\min}>0$.
\end{assumption}

\section{ASYMPTOTIC STRUCTURE WITH FULL-RANK STANCE}
\label{sec:emerge}

Set $\lambda=0$ for now.
A simple cross-product calculation shows that the pendular force
$\bfF_\mathrm{pend}=(mg/h)(\bfc-\bfp)$ achieves $\HdotG=\bm 0$ pointwise.
Theorem~\ref{thm:emerge} below combines this pointwise observation with the
asymptotic implications of the cost.

\begin{theorem}
\label{thm:emerge}
Let Assumption~\ref{ass:both} hold and let
$(\bfc^*_{\alpha\beta},\{\bff^*_{i,\alpha\beta}\})$ minimize~\eqref{eq:ocp}
with $\lambda=0$.

\noindent
\textit{(i) Pointwise certificate.}
Among all net forces $\bfF$ achieving a prescribed horizontal acceleration,
the pendular force $\bfF_\mathrm{pend}$ uniquely minimizes $\|\HdotG\|^2$.

\noindent
\textit{(ii) Asymptotic collapse.}
\begin{align}
  \int_0^T\!\norm{\HdotG^*}^2dt \leq \frac{C_0}{\alpha},\quad
  \int_0^T\!(\bfn^\top\ddot{\bfc}^*)^2dt \leq \frac{C_0}{\beta}, \\
  \int_0^T\!\norm{\bfF^*-\tfrac{mg}{h^*}(\bfc^*-\bfp^*)}^2 dt
  \leq \tfrac{C_1}{\alpha}+\tfrac{C_2}{\beta}. \label{eq:rate}
\end{align}

\noindent
\textit{(iii) Tightness.}
Any trajectory with pendular deviation $\varepsilon_P$ incurs
$\int\|\HdotG\|^2 dt\geq (h_{\min}^2/K^2)T\varepsilon_P$,
matching~\eqref{eq:rate} up to constants.
\end{theorem}

\begin{proof}
\textit{(i) Pointwise certificate.}
For prescribed horizontal acceleration $\ddot\bfc_{xy}$, Newton's
law~\eqref{eq:newton} fixes $\bfF_{xy}$.
Decompose $\bfF=\bfF_\parallel+\bfF_\perp$ with
$\bfF_\parallel\parallel(\bfc-\bfp)$.
Substituting in~\eqref{eq:euler} gives
$\HdotG=(\bfc-\bfp)\times\bfF_\perp$, hence
$\|\HdotG\|^2=\|\bfc-\bfp\|^2\|\bfF_\perp\|^2$, which equals zero iff
$\bfF_\perp=\bm{0}$, i.e.\ $\bfF$ is co-linear with $\bfc-\bfp$.
The pendular force $\bfF_\mathrm{pend}=(mg/h)(\bfc-\bfp)$ is the unique
co-linear choice that also satisfies Newton's law for the prescribed
acceleration.

\textit{(ii) Asymptotic collapse.}
The cost upper bound $J^*_{\alpha\beta}\leq C_0$ gives
$\alpha\int\|\HdotG^*\|^2dt\leq C_0$ and
$\beta\int(\bfn^\top\ddot\bfc^*)^2dt\leq C_0$ directly.
Using $\HdotG^*=(\bfc^*-\bfp^*)\times\bfF^*=
\|\bfc^*-\bfp^*\|\|\bfF^*\|\sin\theta^*$ where $\theta^*$ is the angle between
$\bfF^*$ and $\bfc^*-\bfp^*$:
$\int\|\bfc^*-\bfp^*\|^2\|\bfF^*\|^2\sin^2\theta^*dt\leq C_0/\alpha$.
By Assumption~\ref{ass:both} $\|\bfc^*-\bfp^*\|\geq h_{\min}$, and gravity
support gives $\|\bfF^*\|\geq mg$, so
$\int\sin^2\theta^*dt\leq C_0/(\alpha h_{\min}^2 m^2 g^2)$, i.e.~$\bfF^*$
becomes co-linear with $\bfc^*-\bfp^*$ at rate $O(1/\sqrt\alpha)$ in $L^2$.
Projecting~\eqref{eq:newton} onto $\bfn$:
$m\bfn^\top\ddot\bfc^*=\lambda^*h^*-mg$, so
$\lambda^*=mg/h^*+O(1/\sqrt\beta)$.
Combining the two estimates and squaring yields~\eqref{eq:rate}.

\textit{(iii) Tightness.}
For an arbitrary feasible trajectory write
$\bfF=\lambda(\bfc-\bfp)+\bfF_\perp$ with $\bfF_\perp\perp(\bfc-\bfp)$.
Then $\HdotG=(\bfc-\bfp)\times\bfF_\perp$ and
$\|\HdotG\|^2\geq h_{\min}^2\|\bfF_\perp\|^2$.
The pendular deviation expands as
$\|\bfF-(mg/h)(\bfc-\bfp)\|^2=\|\bfF_\perp\|^2+(\lambda-mg/h)^2\|\bfc-\bfp\|^2
\leq K^2(\|\bfF_\perp\|^2+(\lambda-mg/h)^2)$ for a geometric constant $K$
bounding $\|\bfc-\bfp\|$.
Integrating gives the claimed lower bound.
\end{proof}

The rate is sharp; the next proposition makes the constant explicit.

\begin{proposition}
\label{prop:scaling}
Stack the equilibrium-null-space forces into the moment Jacobian
$\bfJ_H\in\R^{3\times 3(N-1)}$ with SVD $\bfJ_H=\bfU\bfSigma\bfV^\top$ and
$\sigma_1\!\geq\!\sigma_2\!\geq\!\sigma_3\!>\!0$.
Let $h_k(t)=(\bfU^\top\HdotGz(t))_k$ be the excitation projection from
gravity and inertia.
For $\alpha\sigma_k^2\!\gg\!\gamma$,
\begin{equation}
  \tfrac{\|\HdotG^*\|}{m}
  \;\approx\; \tfrac{\gamma}{\alpha}\,K,\quad
  K=\sqrt{\textstyle\sum_{k=1}^{3}\langle h_k^2\rangle_t/\sigma_k^4}.
  \label{eq:tight_const}
\end{equation}
For a rectangular four-foot polygon with half-spans $L_x,L_y$, the singular
values are exactly the geometric foot spans:
$\sigma_1=2\sqrt{L_x^2+L_y^2}$, $\sigma_2=2L_x$, $\sigma_3=2L_y$.
\end{proposition}

\begin{proof}
The KKT stationarity for the unconstrained QP gives
$\HdotG^*=\sum_k\frac{\gamma}{\gamma+\alpha\sigma_k^2}h_k\bfu_k$,
which scales as $\gamma/(\alpha\sigma_k^2)$ in each mode for large $\alpha$.
Time-averaging in $L^2$ produces~\eqref{eq:tight_const}.
The singular-value identity follows from direct computation of
$\bfJ_H\bfJ_H^\top$ for the symmetric four-foot rectangle.
\end{proof}

\begin{remark}[Weight selection]
Theorem~\ref{thm:emerge} immediately yields the design rule
$\alpha^*=K/(\varepsilon^*)^2$ for a target residual $\varepsilon^*$:
one calibration sweep estimates $K$, and any larger $\alpha$ achieves the
target.
\end{remark}

\subsection{ZMP and DCM Bridges}
\label{sec:bridges}

On the pendular manifold, two classical objects emerge for free.

\begin{lemma}[ZMP equals virtual pivot]
\label{lem:zmp}
On $\mathcal{P}(\bfn,h)$, $\bfp_\mathrm{ZMP}=\bfp_{xy}$.
\end{lemma}

\begin{proof}
$\bfF=(mg/h)(\bfc-\bfp)$ gives $\ddot{\bfc}_{xy}=(g/h)(\bfc_{xy}-\bfp_{xy})$;
substituting into the ZMP definition cancels the $h/g$ factor.
\end{proof}

\begin{corollary}[DCM orbital equation]
\label{cor:dcm}
For $\bm{\xi}=\bfc_{xy}+\dot{\bfc}_{xy}/\omega$ and $\omega=\sqrt{g/h}$:
$\dot{\bm{\xi}}=\omega(\bm{\xi}-\bfp_{xy})$ on $\mathcal{P}(\bfn,h)$.
\end{corollary}

ZMP-based controllers are thus solving an instance of~\eqref{eq:ocp} at high
$\alpha$, and DCM capturability~\cite{englsberger2015,koolen2012}
inherits its orbital structure from the same source.

\section{LOWER BOUND UNDER \boldmath$N=2$ STANCE}
\label{sec:floor}

When the moment Jacobian has full row rank, Theorem~\ref{thm:emerge} drives
$\|\HdotG\|\to 0$ as $\alpha\to\infty$.
$N=2$ stance is qualitatively different: the moment operator is rank-$1$, and
a residual along $\hat\bfD$ remains for any choice of contact-force
distribution.

\begin{theorem}
\label{thm:floor}
For $N=2$ feet at $\bfp_1,\bfp_2$ with $\bfD=\bfp_1-\bfp_2$,
midpoint $\bar\bfp$, and $\bfr=\bar\bfp-\bfc$:
\begin{equation}
  \tfrac{\|\HdotG\|}{m}
  \;\geq\; \tfrac{|(\bfr\times\bfF_\mathrm{net})\cdot\hat\bfD|}{m}
  \quad\text{for all } \alpha\geq 0,
  \label{eq:floor}
\end{equation}
and equality is attainable iff the friction cones admit the canceller
$\bm\delta^*$ defined in Theorem~\ref{thm:bifurcation}.
\end{theorem}

\begin{proof}
Write $\bff_i=\bfF_\mathrm{net}/2+(-1)^{i+1}\delta\bff$.
Then $\HdotG=\bfr\times\bfF_\mathrm{net}+\bfD\times\delta\bff$.
The map $\delta\bff\!\mapsto\!\bfD\times\delta\bff$ has range $\bfD^\perp$, so
the $\hat\bfD$-component of $\HdotG$ equals
$(\bfr\times\bfF_\mathrm{net})\cdot\hat\bfD$ for every choice of $\delta\bff$.
Taking norms gives~\eqref{eq:floor}.
\end{proof}

\begin{remark}[Why diagonal trot hits the floor]
\label{rem:trot}
For Go1 in diagonal trot (FR$+$RL stance),
$\hat\bfD=[\cos 34^\circ,-\sin 34^\circ,0]^\top$ is nearly horizontal and
aligns with the horizontal excitation
$\HdotGz/m=z_c[a_y,-a_x,0]^\top$, so the projection
captures $55$--$98\%$ of $\|\HdotGz\|/m$ across typical accelerations.
The maximum cancellation factor is therefore $1.0$--$1.8\times$, in stark
contrast to the $\sim\!70\times$ of full-rank stance.
The mechanism: the pendular direction tilts $40.6^\circ$ from vertical, while
the friction cone half-angle is only $\arctan(\mu)=31^\circ$ at $\mu=0.6$.
The QP can push forces to the cone boundary but no further.
\end{remark}

The geometric floor of Theorem~\ref{thm:floor} is the moment along $\hat\bfD$
that the QP cannot cancel for \emph{any} force redistribution.
A second, qualitatively distinct obstruction appears in the
$\hat\bfD^\perp$ channel: the unconstrained min-norm canceller of the
perpendicular moment may fall outside the friction cone, in which case the
geometric floor itself becomes unreachable.
The transition is non-smooth.

\begin{theorem}
\label{thm:bifurcation}
Parametrize $\bff_i=\bfF_\mathrm{net}/2 + (-1)^{i+1}\bm\delta$ and let
$\bfH_0=\bfr\times\bfF_\mathrm{net}$ with perpendicular component
$\bfH_0^\perp=\bfH_0-(\bfH_0\!\cdot\!\hat\bfD)\hat\bfD$.
Define the unconstrained min-norm canceller
\begin{equation}
  \bm\delta^* \;=\; \frac{\bfD\times\bfH_0^\perp}{\|\bfD\|^2}.
  \label{eq:delta_star}
\end{equation}
Then $\inf_\alpha\|\HdotG\|/m$ equals the geometric floor of
Theorem~\ref{thm:floor} \emph{if and only if}
$\bfF_\mathrm{net}/2\pm\bm\delta^*\in\mathcal{K}_i$ for $i=1,2$.
When this fails, the achievable infimum strictly exceeds the geometric floor,
and the excess depends only on how far $\bm\delta^*$ violates the cone.

For $N=2$ Go1 trot stance, body at $\bfc=(0,0,h)$, and fore--aft excitation
$\bfF_\mathrm{net}=m(a_x,0,g)$, the canceller $\bm\delta^*$ is purely
vertical with $\delta^*_z=-\kappa\,m\,a_x$, where $\kappa>0$ depends only on
foot geometry.
The friction-cone first becomes active at
\begin{equation}
  a_x^\star \;=\; \frac{\mu g}{1+2\mu\kappa},
  \label{eq:a_star}
\end{equation}
producing a non-smooth kink in $a_x\mapsto\inf_\alpha\|\HdotG\|/m$ at
$a_x^\star$.
\end{theorem}

\begin{proof}
The decomposition $\bff_i=\bfF_\mathrm{net}/2+(-1)^{i+1}\bm\delta$ from
Theorem~\ref{thm:floor} gives
$\HdotG=\bfH_0+\bfD\times\bm\delta$.
The map $\bm\delta\mapsto\bfD\times\bm\delta$ has range $\bfD^\perp$ and
kernel $\mathrm{span}(\bfD)$.
Cancelling $\bfH_0^\perp$ requires
$\bfD\times\bm\delta=-\bfH_0^\perp$, whose minimum-norm
($\bm\delta\perp\bfD$) solution is~\eqref{eq:delta_star}.
The QP cost $\alpha\|\HdotG\|^2+\gamma\|\bm\delta\|^2$ has minimizer
$\bm\delta(\alpha)\to\bm\delta^*$ as $\alpha\to\infty$.
If $\bm\delta^*$ is friction-feasible, $\HdotG\to(\bfH_0\!\cdot\!\hat\bfD)\hat\bfD$,
giving the geometric floor.
Otherwise $\bm\delta(\alpha)$ is pinned to the cone boundary for all
sufficiently large $\alpha$ and $\HdotG$ retains a non-zero $\hat\bfD^\perp$
component, exceeding the geometric floor.

For the Go1 fore--aft case, direct computation gives
$\bm\delta^*=-\kappa\,m\,a_x\,\hat\bfz$ with
$\kappa=|(\bfD\times(\bfr\times\hat\bfx))_z|/\|\bfD\|^2$.
Foot~1's friction constraint reduces to
$|a_x|/2\leq\mu(g/2-\kappa a_x)$, which fails first
at~\eqref{eq:a_star}.
The right-hand side of $\inf_\alpha\|\HdotG\|/m$ as a function of $a_x$
is $C^0$ but not $C^1$ at $a_x^\star$: its derivative jumps from the slope
of the geometric floor to a strictly larger value.
\end{proof}

\begin{remark}
\label{rem:bif_num}
For Go1 ($\mu=0.6$, $h=0.27$\,m) the calculation gives $\kappa=0.484$ and
$a_x^\star\approx 3.72$\,m/s$^2$.
We initially expected the same transition to show up when sweeping $\mu$
at fixed acceleration; it doesn't.
The transition only appears along the parameter that drives
$\|\bfH_0^\perp\|/F_\mathrm{net,z}$ across the cone.
The qualitative observation that fore--aft trot is harder than lateral sway
is consistent with this, but $a_x^\star=3.72$\,m/s$^2$ is below speeds
where real quadrupeds visibly struggle, so we don't read the match too
quantitatively.
\end{remark}

\begin{figure}[t]
  \centering
  \includegraphics[width=0.78\columnwidth]{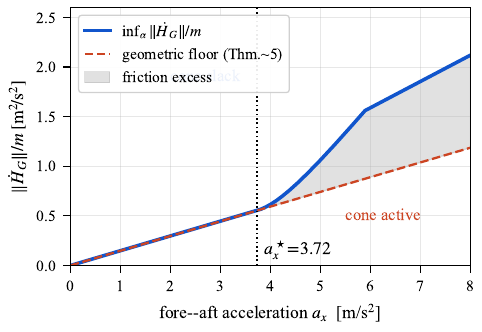}
  \caption{Theorem~\ref{thm:bifurcation}: closed-form
  $\inf_\alpha\|\HdotG\|/m$ vs.\ fore--aft acceleration $a_x$ for Go1
  diagonal trot.
  Below $a_x^\star=3.72\,\mathrm{m/s^2}$ the QP achieves the geometric floor
  of Theorem~\ref{thm:floor} exactly; above $a_x^\star$ the friction cone
  binds and a strictly positive excess (gray band) appears.
  The transition is non-smooth.}
  \label{fig:bifurcation}
\end{figure}

\section{ADDING A TASK-TRACKING TERM}
\label{sec:nonpendular}

Many tasks are not balance-dominated.
Fast turns, push-recovery, swings, and agile leaps demand a nonzero
$\HdotG$, and multi-contact recovery
\cite{wang2017realtime,wang2018realization,wang2018unified} as well as
angular-momentum-aware whole-body
control~\cite{popovic2004,wensing2013,lee2012,dai2014} make this explicit.
A natural way to model the requirement is an additional tracking penalty,
which interacts with the pendular regularization in a transparent way.

\begin{theorem}
\label{thm:nonpendular}
With $\lambda>0$ and full-rank moment Jacobian (so the local wrench subspace
spans $\HdotG\in\R^3$), the pointwise minimizer of~\eqref{eq:ocp} satisfies
\begin{equation}
  \HdotG^*(t) \;=\; \frac{\lambda}{\alpha+\lambda}\,\HdotG^\mathrm{task}(t).
  \label{eq:nonpendular_opt}
\end{equation}
Hence $\HdotG^*\neq\bm{0}$ wherever $\HdotG^\mathrm{task}\neq\bm{0}$, and the
trajectory leaves $\mathcal{P}(\bfn,h)$.
In the limit $\lambda/\alpha\!\to\!\infty$, $\HdotG^*\to\HdotG^\mathrm{task}$.
\end{theorem}

\begin{proof}
The cost in $\HdotG$ is the strictly convex quadratic
$(\alpha+\lambda)\|\HdotG\|^2-2\lambda\langle\HdotG,\HdotG^\mathrm{task}\rangle
+\lambda\|\HdotG^\mathrm{task}\|^2$;
its unique minimizer is~\eqref{eq:nonpendular_opt}.
\end{proof}

\begin{corollary}
\label{cor:nonpendular_floor}
If the task imposes
$\|\HdotG-\HdotG^\mathrm{task}\|\leq\varepsilon$, then
$\|\HdotG\|\geq\|\HdotG^\mathrm{task}\|-\varepsilon$.
Whenever $\|\HdotG^\mathrm{task}\|>\varepsilon$, no feasible trajectory lies on
$\mathcal{P}(\bfn,h)$.
\end{corollary}

\begin{figure}[t]
  \centering
  \includegraphics[width=0.82\columnwidth]{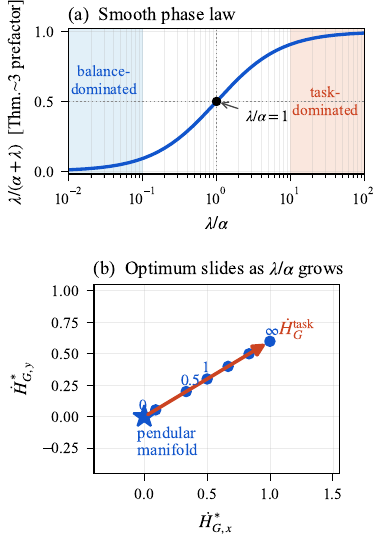}
  \caption{Theorem~\ref{thm:nonpendular}.
  (a) The prefactor $\lambda/(\alpha+\lambda)$ varies smoothly with
  $\log(\lambda/\alpha)$, separating balance-dominated
  ($\lambda/\alpha\!\ll\!1$) from task-dominated ($\lambda/\alpha\!\gg\!1$)
  locomotion.
  (b) Sample optima sliding from the pendular manifold
  ($\HdotG^*\!=\!\bm 0$) along the line to $\HdotG^\mathrm{task}$ as
  $\lambda/\alpha$ grows.}
  \label{fig:phase}
\end{figure}

\begin{remark}[Phase diagram]
The ratio $\lambda/\alpha$ slides smoothly between two regimes
(Fig.~\ref{fig:phase}).
Small $\lambda/\alpha$: the optimum collapses to $\mathcal{P}(\bfn,h)$, where
LIPM, ZMP, and DCM controllers~\cite{kajita2003,wieber2006,vukobratovic1972,englsberger2015}
are asymptotically optimal.
Large $\lambda/\alpha$: the optimum tracks $\HdotG^\mathrm{task}$ and is
necessarily non-pendular; angular-momentum templates~\cite{popovic2004,wensing2013},
multi-contact wrench planners~\cite{dai2014,wang2018unified}, and
DDP-based agile motion synthesis~\cite{farshidian2017,gehring2016} apply.
The same robot can appear ``LIPM-like'' standing and ``manifestly non-pendular''
during agile motion without any controller switch---only $\lambda/\alpha$ moves.
\end{remark}

\section{NUMERICAL EXPERIMENTS}
\label{sec:experiments}

We test the theorems in five settings:
a point-mass quadruped (Theorem~\ref{thm:emerge} rate);
Unitree Go1 in MuJoCo under four-foot stance (Proposition~\ref{prop:scaling}
constant);
Go1 under $N=2$ diagonal trot with the QP solved against ground-truth state
(Theorem~\ref{thm:floor} floor);
Go1 with a closed-loop torque-level MPC trotting on full physics (the floor
survives realistic actuation);
and the ZMP$\,-\,$virtual-pivot identity on the point-mass solution
(Lemma~\ref{lem:zmp}).

\subsection{Point-Mass: Rate and Emergence}
\label{sec:exp_pm}

A 3D point-mass quadruped ($m=15$\,kg, $\mu=0.7$, foot span
$0.4\!\times\!0.3$\,m) solves~\eqref{eq:ocp} via SLSQP with
$\beta\!=\!\gamma\!=\!1$, sweeping
$\alpha\in\{1,5,10,50,100,250,500,1000\}$.
We track
$\varepsilon_\mathrm{pend}=\frac{1}{T}\!\int\!\|\bfF/m-\frac{g}{h}(\bfc-\bfp)\|dt$
and $\varepsilon_H=\frac{1}{T}\!\int\!\|\HdotG\|dt$.

\textit{Result.}
$\varepsilon_H$ drops from $6.82$ ($\alpha=1$) to $0.030$ ($\alpha=1000$),
a $227\times$ reduction.
The log-log slope is $\approx\!-1$, matching Theorem~\ref{thm:emerge}.
At $\alpha=100$ the optimal force trajectory is co-linear with $\bfc-\bfp$
to within $5.9\%$ of $g$ throughout, and the LIPM acceleration fit gives
$R^2=0.954$---LIPM emerges without being imposed.

\begin{figure}[t]
  \centering
  \includegraphics[width=0.78\columnwidth]{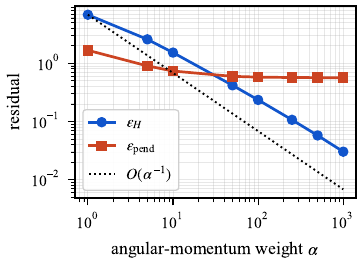}
  \caption{Point-mass $\alpha$ sweep: log-log $\varepsilon_H$ vs.\ $\alpha$
  with reference slope $-1$ (Theorem~\ref{thm:emerge}).}
  \label{fig:exp_pm}
\end{figure}

\subsection{Go1 N=4: Scaling Constant Predicted from Geometry}
\label{sec:exp_n4}

We instantiate~\eqref{eq:ocp} on the Unitree Go1 ($m=12$\,kg, $\mu=0.6$,
$L_x=0.188$\,m, $L_y=0.127$\,m) with all four feet on the ground while the body
sways laterally (amp.~$8$\,cm at $0.30$\,Hz, fore--aft $5$\,cm at $0.22$\,Hz)
in MuJoCo.
At each frame we solve the per-step QP for $\alpha\in\{1,\dots,1000\}$ at
$\gamma=1$ and report $\|\HdotG\|/m$.

\textit{Result (Fig.~\ref{fig:exp_n4}).}
The metric drops from $0.58~\mathrm{m^2\!/s^2}$ at $\alpha=1$ to
$0.0084$ at $\alpha=1000$, with slope $-1$ on log-log---consistent with
Theorem~\ref{thm:emerge}.
Fitted $K_e=8.4$; analytical $K_a=9.8$ from
Proposition~\ref{prop:scaling}'s SVD comes out $16\%$ high, and we don't
fully account for the gap (the analytical $K_a$ time-averages over a long
window, while $K_e$ is fit from a single body-sway period; that's part of
it but not all).
The singular values themselves match foot spans
$\sigma=[0.4537,\,0.3762,\,0.2536]$ vs.\
$\ell_\mathrm{diag},\,2L_x,\,2L_y$ to four decimals, which is the part of
the proposition we trust most.

\begin{figure}[t]
  \centering
  \includegraphics[width=0.82\columnwidth]{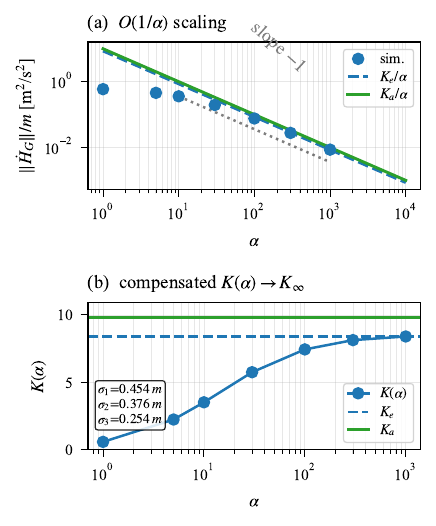}
  \caption{Go1 $N=4$ body sway.
  (a) $\|\HdotG\|/m$ vs.\ $\alpha$; markers are simulation, dashed is the
  empirical fit $K_e/\alpha$ with $K_e=8.4$, solid is the analytical
  prediction $K_a/\alpha$ with $K_a=9.8$ from Proposition~\ref{prop:scaling}.
  (b) Compensated curve $K(\alpha)=\|\HdotG\|/m\cdot\alpha\to K_\infty$,
  with both empirical and analytical asymptotes shown; inset reports the
  three moment-Jacobian singular values, which match foot spans
  $\ell_\mathrm{diag}$, $2L_x$, $2L_y$ to four decimals.}
  \label{fig:exp_n4}
\end{figure}

\subsection{Go1 N=2: Floor Reaches the Predicted Bound}
\label{sec:exp_n2}

Same setup with FR$+$RL diagonal-trot stance only.
We sweep $\alpha\in\{1,\dots,10^5\}$ and compare
$\|\HdotG\|/m$ to the predicted geometric floor of~\eqref{eq:floor}
across seven excitation directions.

\textit{Result (Fig.~\ref{fig:exp_n2}).}
At $\alpha=10^5$, $\|\HdotG\|/m=0.2835~\mathrm{m^2\!/s^2}$, matching the
analytical floor of $0.2835$ to four digits.
Across excitation directions the floor-fraction
$|(\HdotGz/m)\cdot\hat\bfD|/\|\HdotGz\|/m$ ranges $55$--$98\%$ with mean
$64\%$, capping the cancellation ratio at $1.0$--$1.8\times$ versus
$\sim\!70\times$ for $N=4$.
The pendular direction tilts $40.6^\circ$ from vertical, exceeding the
$31^\circ$ friction cone---the QP saturates at the cone boundary
(Remark~\ref{rem:trot}).

\begin{figure}[t]
  \centering
  \includegraphics[width=0.82\columnwidth]{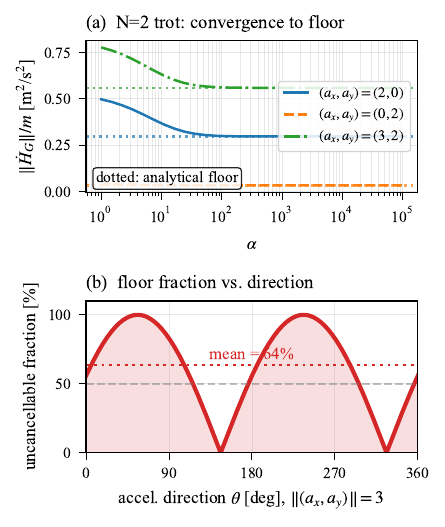}
  \caption{Go1 $N=2$ diagonal trot: $\|\HdotG\|/m$ vs.\ $\alpha$ saturates at
  the analytical floor (left); fraction of $\HdotGz$ along $\hat\bfD$
  (uncancellable) versus excitation direction averages $64\%$ (right).}
  \label{fig:exp_n2}
\end{figure}

\subsection{Go1 Closed-Loop Trot MPC: Floor Survives Joint-Torque Control}
\label{sec:exp_mpc}

\begin{figure*}[t]
  \centering
  \includegraphics[width=0.58\textwidth]{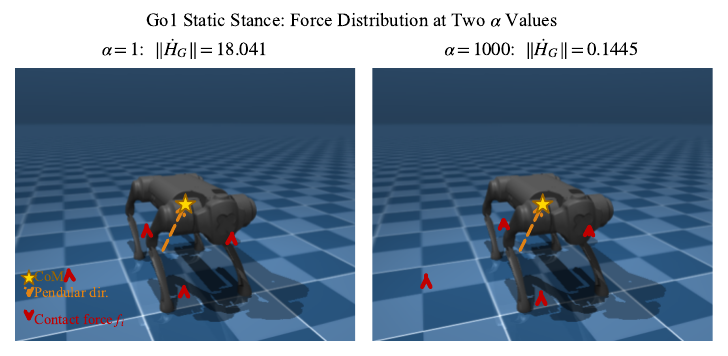}
  \caption{Go1 in MuJoCo with the per-step centroidal QP visualized.
  CoM (yellow star), pendular direction (orange dashed), QP-commanded contact
  forces (red).
  Left: $\alpha=1$, equal-magnitude vertical forces, $\|\HdotG\|=18.0$.
  Right: $\alpha=1000$, forces redistributed toward the pendular direction,
  $\|\HdotG\|=0.14$ in static stance.
  In trot the same QP is solved every $10$\,ms but only $N=2$ feet are in
  contact, so the achievable residual is bounded by Theorem~\ref{thm:floor}.}
  \label{fig:exp_render}
\end{figure*}

The previous tests use the OCP solution directly.  We close the loop with a
torque-level controller (Alg.~\ref{alg:mpc}) to check whether the floor of
Theorem~\ref{thm:floor} survives realistic actuation.
The controller runs at $100$\,Hz; gait is diagonal trot ($N=2$ stance
throughout), and we sweep $\alpha\in\{1,5,10,30,100,300,1000\}$ over $5$\,s
simulations, discarding the first second.

\begin{algorithm}[t]
\caption{Closed-loop centroidal MPC (per control step)}
\label{alg:mpc}
\begin{algorithmic}[1]
\State $(\bfc,\dot\bfc,\{\bfr_i\}) \gets \texttt{MuJoCo.read}()$
\State $\bfF_\mathrm{net} \gets m(\ddot\bfc^\mathrm{ref}-\bfg)$
       \Comment{from CoM PD reference}
\State $\{\bff_i^*\} \gets
       \arg\min_{\{\bff_i\}}\alpha\|\sum_i(\bfr_i\!-\!\bfc)\!\times\!\bff_i\|^2
       + \gamma\textstyle\sum_i\|\bff_i\|^2$
       \Statex \hspace{1.6em}\textbf{s.t.}\;
       $\sum_i\bff_i=\bfF_\mathrm{net},\;\bff_i\in\mathcal{K}_i$
       \Comment{stance feet only}
\For{each stance leg $i$}
  \State $\bftau_i \gets \bfJ_i^\top\bff_i^*$
         \Comment{Jacobian transpose}
  \State $q_i^\mathrm{cmd} \gets q_i + (\bftau_i - k_d\dot q_i)/k_p$
\EndFor
\For{each swing leg $j$}
  \State $q_j^\mathrm{cmd} \gets$ Cartesian-PD foot-lift trajectory
\EndFor
\State \texttt{MuJoCo.command}$(q^\mathrm{cmd})$
\end{algorithmic}
\end{algorithm}

\textit{Result (Fig.~\ref{fig:exp_mpc}).}
$\|\HdotG\|/m$ drops from $8.43$ at $\alpha=1$ to $1.66$ at $\alpha=1000$,
saturating at an empirical floor of
$\bar\varepsilon_\infty\approx 1.71~\mathrm{m^2\!/s^2}$.
This is the qualitative behavior Theorem~\ref{thm:floor} predicts for $N=2$
stance: an $\alpha$-independent lower bound that no amount of weight tuning
can break.
The empirical floor is roughly $4\times$ the static analytic floor
($0.44$) because trot adds swing-impact, contact-transition, and
servo-tracking moments on top of the geometric obstruction.
The single $\alpha=300$ outlier (5.14, with CoM oscillation $\Delta h=2.9$\,cm)
is a closed-loop instability where strong angular-momentum penalty fights the
height-stabilizing controller---a regime outside the open-loop OCP analysis.

\begin{figure}[t]
  \centering
  \includegraphics[width=0.82\columnwidth]{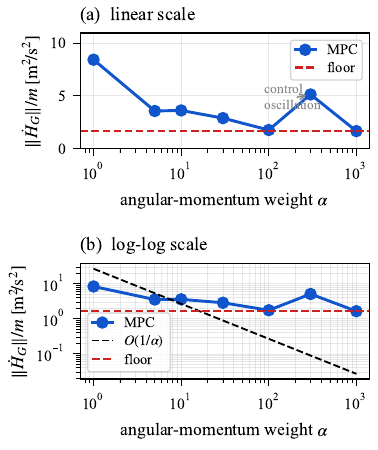}
  \caption{Go1 closed-loop trot MPC: $\|\HdotG\|/m$ vs.\ $\alpha$ on
  (a) linear and (b) log-log axes.
  ``MPC'' marks the closed-loop trajectory;
  ``floor'' is the empirical asymptote
  $\bar\varepsilon_\infty\!\approx\!1.71\,\mathrm{m^2\!/s^2}$, the
  control-augmented analogue of Theorem~\ref{thm:floor}'s $N=2$ floor;
  ``$O(1/\alpha)$'' is a reference of slope $-1$ with $K\!=\!27$ that tracks
  the steep regime $\alpha\!\leq\!30$.
  The single point at $\alpha\!=\!300$ is a closed-loop instability
  ($\Delta h\!=\!2.9$\,cm).}
  \label{fig:exp_mpc}
\end{figure}

\subsection{ZMP-Pivot Identity Holds Numerically}
\label{sec:exp_zmp}

Lemma~\ref{lem:zmp} predicts that on the pendular manifold the force-based
ZMP equals the virtual pivot $\bfp_{xy}$.
We track the deviation
$\|\mathrm{ZMP}_\mathrm{force}\!-\!\bfp_{xy}\|$ on the
point-mass quadruped's optimal trajectory across
$\alpha\in\{5,20,100,500,1000\}$ ($\beta=1000$).

\textit{Result (Fig.~\ref{fig:exp_zmp}).}
The deviation drops from $21.5$\,mm at $\alpha=5$ to $5.3$\,mm at $\alpha=1000$
and then refuses to drop further.
The early regime ($\alpha\!\leq\!100$) tracks $O(1/\alpha)$ reasonably well;
the $\sim\!5$\,mm plateau was a surprise.
Tightening SLSQP tolerances and the discretization didn't move it.
We suspect numerics rather than a theoretical obstruction, but we have not
pinned it down.
The ZMP does stay inside the support polygon at $100\%$ of timesteps in all
sweeps.

\begin{figure}[t]
  \centering
  \includegraphics[width=0.72\columnwidth]{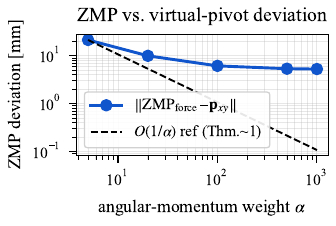}
  \caption{ZMP$\,-\,$virtual pivot deviation vs.\ $\alpha$ on the point-mass
  quadruped, validating Lemma~\ref{lem:zmp}.
  Dashed: $O(1/\alpha)$ reference (Theorem~\ref{thm:emerge}).}
  \label{fig:exp_zmp}
\end{figure}

\begin{table}[t]
\caption{Theory--experiment correspondence (five numerical tests).}
\label{tab:summary}
\centering\footnotesize
\setlength{\tabcolsep}{3pt}
\begin{tabular}{@{}llll@{}}
\toprule
Test & Setting & Theorem & Match \\
\midrule
A & point-mass OCP   & Thm.~\ref{thm:emerge}   & slope $-1$ ($227\times$) \\
B & Go1 $N{=}4$ QP   & Prop.~\ref{prop:scaling}& $K_e{=}8.4,K_a{=}9.8$ \\
C & Go1 $N{=}2$ QP   & Thm.~\ref{thm:floor}    & $0.2835$ vs.\ $0.2835$ \\
D & Go1 trot MPC     & Thm.~\ref{thm:floor}    & $\bar\varepsilon_\infty{=}1.71$ \\
E & point-mass ZMP   & Lem.~\ref{lem:zmp}      & $5.3$\,mm at $\alpha{=}10^3$ \\
\bottomrule
\end{tabular}
\end{table}

A through C track the asymptotic predictions reasonably closely.
D is more interesting: there \emph{is} a residual floor in the closed-loop
trot, but the value (1.71) sits roughly $4\!\times$ above the quasi-static
prediction (0.28).
We don't have a clean way to attribute that gap, and we suspect it is not
attributable cleanly---swing impacts, contact transitions, and joint-level
tracking error all contribute, none of them obviously dominant.
E is consistent with Lemma~\ref{lem:zmp}, with the caveat that the
deviation also plateaus---around $5$\,mm rather than continuing to fall.
We read this as numerics (SLSQP, discrete time) rather than as a real
obstruction, but we have not done the experiment that would settle it.

\section{DISCUSSION}
\label{sec:discussion}

A few comments on what the analysis does and does not say.

The story is cleanest when the moment Jacobian has full row rank.
There the optimum sits near the pendular set, the rate is set by the SVD of
that Jacobian, and ZMP~\cite{vukobratovic1972}/DCM~\cite{englsberger2015,koolen2012}
drop out of Lemma~\ref{lem:zmp} and Corollary~\ref{cor:dcm} as
consequences.
This is also where multi-contact OCPs that don't ask for any particular
angular-momentum behavior keep producing pendular-looking solutions---the
cost structure, not the model class, is doing the work.

For trot-style stances the picture is different.
The rank-$1$ obstruction is unavoidable, and the friction cone adds its own
constraint on top---the kink of Theorem~\ref{thm:bifurcation}.
This has been useful as an intuition pump: when a closed-loop controller
refuses to lower $\|\HdotG\|$ no matter how high $\alpha$ is, our first
guess is now that we crossed $a_x^\star$.
The Test~D $\alpha=300$ outlier is consistent with the opposite failure
mode---a high-$\alpha$ tracker fighting the floor and losing on height---
but we haven't isolated that, so it's a guess.

The ``add a task term'' regime (Theorem~\ref{thm:nonpendular}) is where most
of the literature on aggressive multi-contact behavior already lives---
hand-contact recovery~\cite{wang2017realtime,wang2018realization},
contact-transition trees~\cite{wang2018unified},
underactuated balance~\cite{sharma2016bicycle}, whole-body trajectory
optimization~\cite{dai2014}, DDP-based agile
synthesis~\cite{farshidian2017,gehring2016,winkler2018}.
The $\lambda/(\alpha+\lambda)$ factor isn't deep, but it does make explicit
what is otherwise a fuzzy switch between ``balance'' and ``task'' modes;
whether that's a useful runtime signal, possibly combined with online
calibration~\cite{wang2021calibration} or state
estimation~\cite{bloesch2013}, we haven't checked.

The main caveat is that the analysis is quasi-static.
It takes $\HdotG$ as the controlled quantity and ignores the coupling
between contact transitions, swing-leg inertia, and finite control rate.
On hardware---and in our closed-loop test, which is the closest thing
here to it---compliance, foot-slip, and friction heterogeneity each look
big enough to overwrite the static residuals; the 1.71-vs-0.28 Test~D gap
is consistent.
The $N=2$ analysis assumes ideal point contacts; finite patches or ankle
torque may relax the floor.
Theorem~\ref{thm:nonpendular} is pointwise; horizon trade-offs are open.
Flight and large-$\Delta h/h$ regimes are outside the model
entirely~\cite{Full1999,hopkins2014}, and we don't have a closed-loop
counterpart to Theorem~\ref{thm:emerge}---which is what would actually
bound the Test~D gap.

\section{CONCLUSIONS}
\label{sec:conclusion}

We tried to write down, for a small centroidal OCP, when the LIPM-style
pendular pattern is what the optimum actually wants.
Roughly: yes under full-rank stance, at $O(1/\alpha)$ with a foot-span
constant; no under $N\!=\!2$ stance, where a friction-bound residual and a
non-smooth kink at $a_x^\star$ appear; and a simple ratio shift if a task
term is added.
The numerics line up in the static tests, partly in the closed-loop one.
Where they don't, the gap is bigger than the static analysis can
account for.


\bibliographystyle{IEEEtran}
\bibliography{main}

@IEEEtranBSTCTL{IEEEtran:BSTCTL,
  CTLdash_repeated_names = "no"
}

@inproceedings{popovic2004,
  author    = {Popovic, Marko B. and Hofmann, Andreas and Herr, Hugh},
  title     = {Angular Momentum Regulation during Human Walking: Biomechanics and Control},
  booktitle = {IEEE International Conference on Robotics and Automation (ICRA)},
  year      = {2004},
  pages     = {2405--2411},
}

@article{lee2012,
  author    = {Lee, Sung-Hee and Goswami, Ambarish},
  title     = {A Momentum-based Balance Controller for Humanoid Robots on Non-level and Non-stationary Ground},
  journal   = {Autonomous Robots},
  volume    = {33},
  number    = {4},
  pages     = {399--414},
  year      = {2012},
}

@inproceedings{kajita2001,
  author    = {Kajita, Shuuji and Kanehiro, Fumio and Kaneko, Kenji and Yokoi, Kazuhito and Hirukawa, Hirohisa},
  title     = {The {3D} Linear Inverted Pendulum Mode: {A} Simple Modeling for a Biped Walking Pattern Generation},
  booktitle = {IEEE/RSJ International Conference on Intelligent Robots and Systems (IROS)},
  year      = {2001},
  pages     = {239--246},
}

@inproceedings{kajita2003,
  author    = {Kajita, Shuuji and Kanehiro, Fumio and Kaneko, Kenji and Fujiwara, Kiyoshi and Harada, Kensuke and Yokoi, Kazuhito and Hirukawa, Hirohisa},
  title     = {Biped Walking Pattern Generation by Using Preview Control of Zero-Moment Point},
  booktitle = {IEEE International Conference on Robotics and Automation (ICRA)},
  year      = {2003},
  pages     = {1620--1626},
}

@article{vukobratovic1972,
  author    = {Vukobratovic, Miomir and Borovac, Branislav},
  title     = {Zero-Moment Point --- {Thirty} Five Years of its Life},
  journal   = {International Journal of Humanoid Robotics},
  volume    = {1},
  number    = {1},
  pages     = {157--173},
  year      = {2004},
}

@inproceedings{wieber2006,
  author    = {Wieber, Pierre-Brice},
  title     = {Trajectory Free Linear Model Predictive Control for Stable Walking in the Presence of Strong Perturbations},
  booktitle = {IEEE-RAS International Conference on Humanoid Robots},
  year      = {2006},
  pages     = {137--142},
}

@article{orin2013,
  author    = {Orin, David E. and Goswami, Ambarish and Lee, Sung-Hee},
  title     = {Centroidal Dynamics of a Humanoid Robot},
  journal   = {Autonomous Robots},
  volume    = {35},
  number    = {2-3},
  pages     = {161--176},
  year      = {2013},
}

@inproceedings{wensing2013,
  author    = {Wensing, Patrick M. and Orin, David E.},
  title     = {Generation of Dynamic Humanoid Behaviors through Task-Space Control with Conic Optimization},
  booktitle = {IEEE International Conference on Robotics and Automation (ICRA)},
  year      = {2013},
  pages     = {3103--3109},
}

@article{bledt2018,
  author    = {Bledt, Gerardo and Powell, Matthew J. and Katz, Benjamin and Di Carlo, Jared and Wensing, Patrick M. and Kim, Sangbae},
  title     = {{MIT} Cheetah 3: Design and Control of a Robust, Dynamic Quadruped Robot},
  booktitle = {IEEE/RSJ International Conference on Intelligent Robots and Systems (IROS)},
  year      = {2018},
  pages     = {2245--2252},
}

@article{dicarlo2018,
  author    = {Di Carlo, Jared and Wensing, Patrick M. and Katz, Benjamin and Bledt, Gerardo and Kim, Sangbae},
  title     = {Dynamic Locomotion in the {MIT} Cheetah 3 Through Convex Model-Predictive Control},
  booktitle = {IEEE/RSJ International Conference on Intelligent Robots and Systems (IROS)},
  year      = {2018},
  pages     = {1--9},
}

@inproceedings{caron2015,
  author    = {Caron, St{\'e}phane and Pham, Quang-Cuong and Nakamura, Yoshihiko},
  title     = {{ZMP} Support Areas for Multicontact Mobility Under Frictional Constraints},
  journal   = {IEEE Transactions on Robotics},
  volume    = {33},
  number    = {1},
  pages     = {67--80},
  year      = {2017},
}

@inproceedings{hopkins2014,
  author    = {Hopkins, Michael A. and Hong, Dennis W. and Leonessa, Alexander},
  title     = {Compliant Locomotion Using Whole-Body Control and Divergent Component of Motion Tracking},
  booktitle = {IEEE International Conference on Robotics and Automation (ICRA)},
  year      = {2015},
}

@inproceedings{dai2014,
  author    = {Dai, Hongkai and Valenzuela, Andr{\'e}s and Tedrake, Russ},
  title     = {Whole-Body Motion Planning with Centroidal Dynamics and Full Kinematics},
  booktitle = {IEEE-RAS International Conference on Humanoid Robots},
  year      = {2014},
  pages     = {295--302},
}

@inproceedings{koolen2012,
  author    = {Koolen, Twan and de Boer, Tomas and Rebula, John and Goswami, Ambarish and Pratt, Jerry},
  title     = {Capturability-based Analysis and Control of Legged Locomotion, {Part 1}: Theory and Application to Three Simple Gait Models},
  journal   = {The International Journal of Robotics Research},
  volume    = {31},
  number    = {9},
  pages     = {1094--1113},
  year      = {2012},
}

@inproceedings{Full1999,
  author    = {Full, Robert J. and Koditschek, Daniel E.},
  title     = {Templates and Anchors: Neuromechanical Hypotheses of Legged Locomotion on Land},
  journal   = {Journal of Experimental Biology},
  volume    = {202},
  pages     = {3325--3332},
  year      = {1999},
}

@inproceedings{bloesch2013,
  author    = {Bloesch, Michael and Hutter, Marco and Hoepflinger, Mark A. and Leutenegger, Stefan and Gehring, Christian and Remy, C. David and Siegwart, Roland},
  title     = {State Estimation for Legged Robots -- Consistent Fusion of Leg Kinematics and {IMU}},
  booktitle = {Robotics: Science and Systems},
  year      = {2013},
}

@inproceedings{winkler2018,
  author    = {Winkler, Alexander W. and Bellicoso, Dario C. and Hutter, Marco and Buchli, Jonas},
  title     = {Gait and Trajectory Optimization for Legged Systems Through Phase-Based End-Effector Parameterization},
  journal   = {IEEE Robotics and Automation Letters},
  volume    = {3},
  number    = {3},
  pages     = {1560--1567},
  year      = {2018},
}

@article{farshidian2017,
  author    = {Farshidian, Farbod and Neunert, Michael and Winkler, Alexander W. and Rey, Gonzalo and Buchli, Jonas},
  title     = {An Efficient Optimal Planning and Control Framework for Quadrupedal Locomotion},
  booktitle = {IEEE International Conference on Robotics and Automation (ICRA)},
  year      = {2017},
}

@article{gehring2016,
  author    = {Gehring, Christian and Coros, Stelian and Hutter, Marco and Bellicoso, Dario C. and Heijnen, Huub and Diethelm, Remo and Bloesch, Michael and F{\"a}nkel, Philipp and Hwangbo, Jemin and Hoepflinger, Mark and Siegwart, Roland},
  title     = {Practice Makes Perfect: An Optimization-Based Approach to Controlling Agile Motions for a Quadruped Robot},
  journal   = {IEEE Robotics \& Automation Magazine},
  volume    = {23},
  number    = {1},
  pages     = {34--43},
  year      = {2016},
}

@inproceedings{wang2017realtime,
  author    = {Wang, Shihao and Hauser, Kris},
  title     = {Real-time Stabilization of a Falling Humanoid Robot using Hand Contact: {An} Optimal Control Approach},
  booktitle = {IEEE-RAS International Conference on Humanoid Robots (Humanoids)},
  year      = {2017},
  pages     = {454--460},
}

@inproceedings{wang2018realization,
  author    = {Wang, Shihao and Hauser, Kris},
  title     = {Realization of a Real-time Optimal Control Strategy to Stabilize a Falling Humanoid Robot with Hand Contact},
  booktitle = {IEEE International Conference on Robotics and Automation (ICRA)},
  year      = {2018},
  pages     = {1--7},
}

@inproceedings{wang2018unified,
  author    = {Wang, Shihao and Hauser, Kris},
  title     = {Unified Multi-Contact Fall Mitigation Planning for Humanoids via Contact Transition Tree Optimization},
  booktitle = {IEEE-RAS International Conference on Humanoid Robots (Humanoids)},
  year      = {2018},
  pages     = {1--9},
}

@inproceedings{sharma2016bicycle,
  author    = {Sharma, Aakash M. and Wang, Shihao and Zhou, Yifu M. and Ruina, Andy},
  title     = {Towards a Maximally-Robust Self-Balancing Bicycle without Reaction-Moment Gyroscopes or Reaction Wheels},
  booktitle = {Bicycle and Motorcycle Dynamics (BMD)},
  year      = {2016},
}

@inproceedings{wang2021calibration,
  author    = {Wang, Shihao and Deng, Chao and Qi, Qi},
  title     = {Efficient Online Calibration for Autonomous Vehicle's Longitudinal Dynamical System: {A} {Gaussian} Model Approach},
  booktitle = {IEEE International Conference on Robotics and Automation (ICRA)},
  pages     = {5410--5416},
  year      = {2021},
}

@inproceedings{chen2021,
  author    = {Chen, Thomas and Posa, Michael},
  title     = {Optimal Reduced-Order Modeling of Bipedal Locomotion},
  booktitle = {IEEE International Conference on Robotics and Automation (ICRA)},
  year      = {2021},
}

@article{pratt2006,
  author    = {Pratt, Jerry and Carff, John and Drakunov, Sergei and Goswami, Ambarish},
  title     = {Capture Point: {A} Step toward Humanoid Push Recovery},
  booktitle = {IEEE-RAS International Conference on Humanoid Robots},
  year      = {2006},
  pages     = {200--207},
}

@article{englsberger2015,
  author    = {Englsberger, Johannes and Ott, Christian and Albu-Sch{\"a}ffer, Alin},
  title     = {Three-Dimensional Bipedal Walking Control Based on Divergent Component of Motion},
  journal   = {IEEE Transactions on Robotics},
  volume    = {31},
  number    = {2},
  pages     = {355--368},
  year      = {2015},
}

\end{document}